%% file: main.tex
\PassOptionsToPackage{table}{xcolor}
\documentclass[sigconf]{acmart}

\AtBeginDocument{%
  }

\setcopyright{none}
\copyrightyear{2026}
\renewcommand\footnotetextcopyrightpermission[1]{}
\acmConference[ICAIF '26 Workshop]{Financial AI Security, Privacy, and Safety}{November 2026}{Milan, Italy}
\acmYear{2026}

\usepackage[table]{xcolor}
\usepackage{amsmath,bm,booktabs,graphicx}
\usepackage{tikz}
\usetikzlibrary{fit,positioning,arrows.meta}
\usepackage{algorithm}
\usepackage{algpseudocode}

\usepackage{float}
\usepackage{stfloats}
\usepackage{subcaption}

\definecolor{forestgreen}{rgb}{0.13, 0.40, 0.13}
\newcommand{\grad}[1]{%
  \ifdim #1pt>98pt
    \cellcolor{forestgreen!50!black}\textcolor{white}{#1}%
  \else\ifdim #1pt>94pt
    \cellcolor{forestgreen!100}\textcolor{white}{#1}%
  \else\ifdim #1pt>84pt
    \cellcolor{forestgreen!85}\textcolor{black}{#1}%
  \else\ifdim #1pt>74pt
    \cellcolor{forestgreen!70}\textcolor{black}{#1}%
  \else\ifdim #1pt>70pt
    \cellcolor{forestgreen!55}\textcolor{black}{#1}%
  \else\ifdim #1pt>64pt
    \cellcolor{forestgreen!40}\textcolor{black}{#1}%
  \else\ifdim #1pt>54pt
    \cellcolor{forestgreen!25}\textcolor{black}{#1}%
  \else\ifdim #1pt>44pt
    \cellcolor{forestgreen!10}\textcolor{black}{#1}%
  \else
    \cellcolor{forestgreen!6}\textcolor{black}{#1}%
  \fi\fi\fi\fi\fi\fi\fi\fi
}

\newtheorem{theorem}{Theorem}

\begin{document}
\title{Federated Layer-Wise Exchange for Cross-Client Money Laundering Pattern Detection}

\author{Selin Ceydeli}
\affiliation{%
  \institution{Delft University of Technology}
  \city{Delft}
  \country{Netherlands}}
\email{S.Ceydeli@student.tudelft.nl}

\author{Rui Wang}
\affiliation{%
  \institution{Delft University of Technology}
  \city{Delft}
  \country{Netherlands}}
\email{R.Wang-8@tudelft.nl}

\author{Kubilay Atasu}
\affiliation{%
  \institution{Delft University of Technology}
  \city{Delft}
  \country{Netherlands}}
\email{Kubilay.Atasu@tudelft.nl}

\renewcommand{\shortauthors}{Trovato et al.}

\begin{abstract}
Subgraph pattern detection is a fundamental graph mining task central to anti-money laundering (AML) in finance. However, state-of-the-art graph neural network (GNN)-based methods assume centralized graph access, impractical when transaction data cannot be shared across institutions. In practice, graphs are distributed across clients, where local GNN training diverges from centralized learning, creating representation and adjoint equivalence gaps that induce a parameter-gradient gap. We formalize this as a structural observability problem, where AML patterns crossing subgraph boundaries become locally unidentifiable. To bridge these gaps, we propose a layer-wise exchange framework: clients synchronize node representations each forward pass layer and adjoints each backward pass, without exposing raw features or labels. Under a shared-parameter assumption, we prove node representations and parameter gradients are equivalent to those of centralized GNN training. Experiments on directed multigraphs with money-laundering patterns show layer-wise exchange performs best with per-step synchronization and remains effective under practical per-epoch synchronization. Our results show layer-wise exchange and parameter synchronization to be complementary, supporting distributed pattern detection in privacy-constrained financial settings.
\end{abstract}



\maketitle

\input{sections/1introduction}
\input{sections/2related-work}
\input{sections/3problem-formulation}
\input{sections/4layerwise-exchange}
\input{sections/5experiments}

\input{sections/6conclusion}

\bibliographystyle{ACM-Reference-Format}
\bibliography{references}

\appendix
\input{sections/appendix-full-problem-formulation}
\input{sections/appendix-proof-of-theorems}
\input{sections/appendix-experimental-setting}
\input{sections/appendix-full-results}
\input{sections/appendix-discussion}

\end{document}

%% file: sections/1introduction.tex
\section{Introduction}
Anti-money laundering (AML) is a high-stakes application of subgraph pattern detection in financial transaction networks \cite{li2020flowscope}. Suspicious transaction patterns can span multiple financial institutions, but legal, privacy, and confidentiality constraints often make sharing transaction data infeasible, so institutions must identify cross-institution patterns without a centrally available graph. Graph neural networks (GNNs) \cite{wu2022graph} can learn substructures and node representations directly from data, but state-of-the-art methods for subgraph pattern detection typically assume centralized access to the entire graph \cite{vignac2020building, you2021identity, huang2023i2gnn, eliasof2023rfp, tahmasebi2023power, egressy2023provably_powerful}.

\input{figures/aml_figure}

In practice, financial transaction graphs are distributed across a multitude of financial institutions.
A natural solution to collaborative model training across institutions without raw data sharing is federated learning (FL),  where clients train GNNs locally and aggregate parameters through a trusted server~\cite{McMahan2017_FedAvg}, mitigating the privacy and confidentiality constraints above. However, parameter aggregation alone does not resolve the key obstacle in cross-client pattern detection: each client sees only part of any subgraph pattern spanning multiple clients, which we call the \emph{structural observability problem}. This causes client representations to diverge from what a centralized model would compute (a \emph{representation equivalence gap}), and the absence of cross-client adjoint contributions yields an \emph{adjoint equivalence gap} and, consequently, a \emph{parameter-gradient gap}, rendering cross-client patterns  locally unidentifiable.


Prior work exchanges embeddings for subgraph FL either once before training \cite{Yao2023_FedGCN,lei2023federated} or after each epoch \cite{Naman2025_RemoteEmbeddings,Li2024_HistoricEmbeddingFedGraph}, but not layer-by-layer, and without returning cross-client adjoints to recover centralized gradients.
We target high-stakes, cross-silo settings such as cross-institution AML detection, where missed laundering activity risks severe regulatory and reputational harm and false positives trigger costly investigations. Therefore, maximizing accuracy is paramount, and the stable, high-bandwidth links between institutions make the added communication cost worthwhile.

Our main contribution is a layer-wise exchange framework that synchronizes remote nodes' representations during each GNN forward pass and the corresponding adjoints during the backward pass, addressing both equivalence gaps. Under identical GNN weights across clients during a given forward/backward pass with per-layer activation and gradient exchange at cross-client edges, we establish \emph{representation equivalence} (Theorem~\ref{thm:representation}) and \emph{gradient equivalence} (Theorem~\ref{thm:gradient}). In practical FL, where parameters synchronize only every few epochs rather than every step, exact equivalence is not guaranteed, yet exchanging representations and adjoints at partition boundaries still substantially reduces lost information.

Furthermore, our experiments isolate the contributions of forward exchange, backward exchange, and federated parameter aggregation on multigraphs with well-known money laundering patterns \cite{Altman2023_RealisticAML} such as directed cycles, scatter-gather flows, and directed bicliques (see Figure~\ref{fig:aml_patterns}). Forward exchange and parameter aggregation prove complementary rather than interchangeable: alone, forward exchange closes only part of the performance gap, while their combination recovers most of it. 
Backward exchange provides further improvements, especially under dense cross-client connectivity.

%% file: figures/aml_figure.tex
\begin{figure}[t]
      \centering
        \resizebox{1.0\linewidth}{!}{%
          \begin{tikzpicture}[
              scale=0.9,
              >=latex,
              semithick,
              own/.style={circle, draw, minimum size=7.5mm, inner sep=0.5pt, font=\small},
              client/.style={draw, rounded corners=4pt, inner sep=4.5mm},
              ie/.style={->},
              ce/.style={->, dashed, gray!65!black}
          ]
    
          \begin{scope}
            \node[own, fill=blue!18]   (ca) at (0,    0.7) {$a$};
            \node[own, fill=orange!22] (cb) at (1.5,  0.7) {$b$};
            \node[own, fill=orange!22] (cc) at (1.5, -0.7) {$c$};
            \node[own, fill=blue!18]   (cd) at (0,   -0.7) {$d$};
    
            \draw[ce] (ca) -- (cb);
            \draw[ie] (cb) -- (cc);
            \draw[ce] (cc) -- (cd);
            \draw[ie] (cd) -- (ca);
    
            \node[client, fit=(ca)(cd), label={[font=\bfseries\small]above:Inst.~1}] {};
            \node[client, fit=(cb)(cc), label={[font=\bfseries\small]above:Inst.~2}] {};
    
            \node[font=\bfseries\small] at (0.75, -2.3) {(a) Directed cycle};
          \end{scope}
    
          \begin{scope}[xshift=5cm]
            \node[own, fill=blue!18]   (ss)  at (0,    0) {$s$};
            \node[own, fill=blue!18]   (si1) at (1.5,  1) {$i_1$};
            \node[own, fill=blue!18]   (si2) at (1.5,  0) {$i_2$};
            \node[own, fill=blue!18]   (si3) at (1.5, -1) {$i_3$};
            \node[own, fill=orange!22] (st)  at (3,    0) {$t$};
    
            \draw[ie] (ss)  -- (si1);
            \draw[ie] (ss)  -- (si2);
            \draw[ie] (ss)  -- (si3);
            \draw[ce] (si1) -- (st);
            \draw[ce] (si2) -- (st);
            \draw[ce] (si3) -- (st);
    
            \node[client, fit=(ss)(si1)(si2)(si3),
                  label={[font=\bfseries\small]above:Inst.~1}] {};
            \node[client, fit=(st),
                  label={[font=\bfseries\small]above:Inst.~2}] {};
    
            \node[font=\bfseries\small] at (1.5, -2.3) {(b) Scatter--gather};
          \end{scope}
    
          \begin{scope}[xshift=10.4cm]
            \node[own, fill=blue!18]   (bl1) at (0,    0.7) {$l_1$};
            \node[own, fill=blue!18]   (bl2) at (0,   -0.7) {$l_2$};
            \node[own, fill=orange!22] (br1) at (1.8,  1.0) {$r_1$};
            \node[own, fill=orange!22] (br2) at (1.8,  0)   {$r_2$};
            \node[own, fill=orange!22] (br3) at (1.8, -1.0) {$r_3$};
    
            \draw[ce] (bl1) -- (br1);
            \draw[ce] (bl1) -- (br2);
            \draw[ce] (bl1) -- (br3);
            \draw[ce] (bl2) -- (br1);
            \draw[ce] (bl2) -- (br2);
            \draw[ce] (bl2) -- (br3);
    
            \node[client, fit=(bl1)(bl2),
                  label={[font=\bfseries\small]above:Inst.~1}] {};
            \node[client, fit=(br1)(br2)(br3),
                  label={[font=\bfseries\small]above:Inst.~2}] {};
    
            \node[font=\bfseries\small] at (0.9, -2.3) {(c) Directed biclique};
          \end{scope}
    
          \end{tikzpicture}%
      }

      \caption{Three money laundering patterns spanning two institutions (Inst. 1, Inst. 2). Solid edges denote intra-institution transfers, while dashed edges cross institutional boundaries.}
      \label{fig:aml_patterns}

      \vspace{-0.1in}
      
  \end{figure}

%% file: sections/2related-work.tex
\section{Related Work}
\label{sec:related-work}

\noindent \textbf{Neighborhood augmentation and graph reconstruction.}
Prior work addresses missing boundary information through synthetic neighbor generation, global graph reconstruction, and privacy-preserving graph expansion. FedSAGE \citep{Zhang2021_SubgraphFedLearning_FedSAGE} and FedDEP \citep{Zhang2024_DeepEfficient} generate synthetic remote-node features or multi-hop embeddings locally on client subgraphs, without sharing the true remote features. FedGGR \citep{Liu2023_FedGGR} and FedStruct \citep{Aliakbari2024DecoupledSF_FedStruct} use global structural information to capture the missing cross-client connectivity. FedPerGNN \citep{wu2022federated} uses a trusted server to match encrypted IDs and return anonymized neighbor embeddings for recommendation. These approaches provide missing-neighbor context through approximation, global structure, or static embeddings. \\

\noindent \textbf{One-time exchange.}
FedGCN \citep{Yao2023_FedGCN} performs a one-time exchange of aggregated multi-hop neighbor features that are reused throughout training. Similarly, FedCog \citep{lei2023federated} decouples local graphs into internal and border parts and exchanges precomputed node representations. Both FedGCN and FedCog compute node representations only once, which reduces communication but prevents them from reflecting evolving layer-wise representations. In contrast, our framework exchanges node embeddings at every layer, keeping them synchronized with model computation. \\
        
\noindent \textbf{Continuous embedding exchange.}
OptimES \citep{Naman2025_RemoteEmbeddings} introduces an embedding server that exchanges node representations at every training epoch, together with systems optimizations to reduce communication and memory overhead. FedAAS \citep{Li2024_HistoricEmbeddingFedGraph} further reduces overhead by caching historical embeddings and selectively synchronizing stale representations. Both focus on communication efficiency and scalability, but neither establishes when exchanged embeddings can recover the node representations produced by centralized training on the full graph.

%% file: sections/3problem-formulation.tex
\section{Problem Formulation}
\label{sec:problem-formulation}

Let $G=(V,E)$ be a global graph with node features $\mathbf{x}:V\to\mathbb{R}^{d}$ and edge features $\mathbf{e}:E\to\mathbb{R}^{d_e}$, distributed across $C$ clients. Client $c$ owns $V_c^{\mathrm{own}}$, where $V=\biguplus_{c=1}^{C}V_c^{\mathrm{own}}$, and $E_c=\{(u,v)\in E:u,v\in V_c^{\mathrm{own}}\}$ denotes its owned edges. Let $V_c^{\mathrm{rem}}$ denote nodes owned by other clients that are
neighbors of nodes in $V_c^{\mathrm{own}}$, and let
$E_c^{\mathrm{rem}}=\{(u,v)\in E: u\in V_c^{\mathrm{own}}, v\in V_c^{\mathrm{rem}}
\text{ or } u\in V_c^{\mathrm{rem}}, v\in V_c^{\mathrm{own}}\}$
denote the cross-client edges incident to $V_c^{\mathrm{own}}$. Client $c$'s
\emph{extended subgraph} is
$G_c=(V_c^{\mathrm{own}}\cup V_c^{\mathrm{rem}},\, E_c\cup E_c^{\mathrm{rem}})$.
An owned node incident to a cross-client edge is a boundary node. Client $c$
knows which remote nodes and cross-client edges are present (i.e., $G_c$ is
available), but not remote-node features or labels.

Consider a shared $L$-layer message-passing GNN, with layer index $l\in\{0,\ldots,L\}$. The layer-$0$ representation $\mathbf{h}_v^{(0)}$ is derived from $\mathbf{x}_v$ according to the GNN architecture. Let $\mathbf{h}_v^{(l,\mathrm{cent})}$ and $\mathbf{h}_{v,c}^{(l)}$ denote the centralized and client-local representations, respectively, and let $\boldsymbol{\delta}_v^{(l,\mathrm{cent})}$ and $\boldsymbol{\delta}_{v,c}^{(l)}$ denote the corresponding adjoints, or derivatives of the respective objectives with respect to those representations. For the adjoint recursion, $\mathcal{N}(v)$ denotes the other nodes whose next-layer representations depend on $\mathbf{h}_v^{(l)}$; the self-dependency is represented separately by $\{v\}$. Without cross-client exchange, client $c$ computes its objective over owned nodes, substituting $\mathbf{h}_{\mathrm{rem}}=\mathbf{0}$ for unavailable remote representations. The explicit message-passing update is given in Appendix~\ref{app:problem-formulation}.

\subsection{Structural Observability Problem}
\label{sec:structural-obs-problem-def}

In FL, missing remote features prevent client $c$ from computing the true layer-$0$ representations of nodes in $V_c^{\mathrm{rem}}$, while local backpropagation omits adjoint contributions from losses on other clients. Local representations and adjoints may therefore deviate from their centralized counterparts. We formalize this structural observability challenge through the forward representation equivalence gap and the backward adjoint equivalence gap.

\paragraph{Forward representation equivalence gap.}
The layer-$l$ forward representation equivalence gap at an owned node $v\in V_c^{\mathrm{own}}$ is $\Delta_{v,c}^{(l)}=\mathbf{h}_v^{(l,\mathrm{cent})}-\mathbf{h}_{v,c}^{(l)}$. In particular, $\Delta_{v,c}^{(L)}$ measures the deviation at the final GNN layer used for prediction. This gap isolates the effect of missing remote features and representations from the effect of missing cross-client edges: the client sees the connectivity, but not what is on the other end. Since the update is recursive, the effect of substituting $\mathbf{h}_{\mathrm{rem}}=\mathbf{0}$ can propagate across layers. After $L$ layers, $\mathbf{h}_{v,c}^{(L)}$ depends on (i) the features of owned nodes within $v$'s $L$-hop neighborhood in the owned subgraph $(V_c^{\mathrm{own}},E_c)$ and (ii) the placeholder inserted at each cross-client edge encountered within $L$ hops. For boundary nodes, $\Delta_{v,c}^{(L)}$ may be nonzero; the discrepancy can propagate to owned nodes several hops away.

\paragraph{Backward adjoint equivalence gap.}
The backward gap concerns adjoints propagated in reverse through message passing and therefore arises only during training. For $l\in\{0,\ldots,L-1\}$, centralized backpropagation gives
\begin{equation}
\boldsymbol{\delta}_v^{(l,\mathrm{cent})}
=
\sum_{u\in\{v\}\cup\mathcal{N}(v)}
\left(
\frac{\partial \mathbf{h}_u^{(l+1,\mathrm{cent})}}
     {\partial \mathbf{h}_v^{(l,\mathrm{cent})}}
\right)^{\!\top}
\boldsymbol{\delta}_u^{(l+1,\mathrm{cent})}.
\label{eq:centralized-adjoint}
\end{equation}

If $v$ is owned by client $c$ and a dependent node $u\in\mathcal{N}(v)$ is owned by another client $c'$, losses associated with $u$ may contribute to the centralized adjoint of $v$. For this cross-client dependency, denote the partial adjoint that client $c'$ computes for remote node $v$ by $\boldsymbol{\delta}_{v\leftarrow c'}^{(l,\mathrm{part})}$. Without backward exchange, this contribution is not returned to owner $c$ and is therefore omitted from $\boldsymbol{\delta}_{v,c}^{(l)}$. We define the layer-$l$ adjoint equivalence gap at an owned node $v\in V_c^{\mathrm{own}}$ as $\Gamma_{v,c}^{(l)}=\boldsymbol{\delta}_v^{(l,\mathrm{cent})}-\boldsymbol{\delta}_{v,c}^{(l)}$.

%% file: sections/4layerwise-exchange.tex
\section{Layer-Wise Exchange Framework}
We address the structural observability problem with a bidirectional layer-wise exchange framework: forward exchange synchronizes node representations across partition boundaries to close the forward representation gap, and backward exchange returns the corresponding partial adjoints to close the backward adjoint gap.

\paragraph{Central assumption of the theoretical analysis:}
All clients share identical parameters $\Theta$ throughout a given forward and backward computation, which is a condition required by Theorem~\ref{thm:representation} for the forward pass and, additionally, by Theorem~\ref{thm:gradient} for the backward pass. In FL this is achieved by aggregating local model updates. More generally, any mechanism that ensures parameter equivalence is sufficient for the theoretical guarantees to hold. 
Importantly, the forward and backward exchange mechanisms presented in this section remain applicable even when this assumption and the corresponding theoretical guarantees are not satisfied.

\subsection{Forward Exchange}
Before computing layer $l+1$ (see Figure~\ref{fig:exchange}, red arrows), the owner of every boundary node $u$ sends its current representation $\mathbf{h}^{(l)}_{u,c'}$ to every client $c$ such that $u \in V^{\mathrm{rem}}_c$, which sets $\mathbf{h}^{(l)}_{u,c} \leftarrow \mathbf{h}^{(l)}_{u,c'}$.
\input{figures/exchange_figure}
At layer $0$, each owner applies the input embedding locally, $\mathbf{h}_{v,c}^{(0)}=\phi_{\theta^{\mathrm{emb}}}(\mathbf{x}_v)$, and transmits only the derived representation. Raw node features and labels never leave their owner.

\begin{theorem}[Representation Equivalence]
\label{thm:representation}
Under the shared-parameter assumption, for every client $c$, every owned node
$v\in V_c^{\mathrm{own}}$, and every layer
$l\in\{0,\ldots,L\}, \mathbf{h}^{(l)}_{v,c}=\mathbf{h}^{(l,\mathrm{cent})}_v.$
Then we have $\Delta^{(l)}_{v,c}=0.$
\end{theorem}
\begin{proof}
Please see Appendix~\ref{app:representation-equivalence-proof}.
\end{proof}

\paragraph{Loss equivalence.}
Client $c$ evaluates
$
L_c(\Theta)=
\sum_{v\in V_c^{\mathrm{own}}}
\ell\!\left(\mathbf{h}^{(L)}_{v,c},y_v\right).
$
Since the owned-node sets partition $V$, Theorem~\ref{thm:representation} gives
$
\sum_{c=1}^{C}L_c(\Theta)
=
L^{\mathrm{cent}}(\Theta):
$
the distributed losses jointly define the centralized objective.

\subsection{Backward Exchange}
Through layer $l$ (see Figure~\ref{fig:exchange}, blue arrows), each client holding $u$ as a remote node returns its partial adjoint $\boldsymbol{\delta}_{u \leftarrow c}^{(l,\mathrm{part})}$ to $u$'s owner; all clients complete layer $l$'s exchanges before backpropagating to layer $l{-}1$.

After backpropagating through layer $l$, client $c$ forms, for every locally materialized node $u$, the partial adjoint contributed by its own owned dependents of $u$:
$$
\boldsymbol{\delta}_{u \leftarrow c}^{(l,\mathrm{part})}
=
\sum_{v \,\in\, V_c^{\mathrm{own}} \,\cap\, (\{u\}\cup\mathcal{N}(u))}
\left(
\frac{\partial \mathbf{h}^{(l+1)}_{v,c}}
     {\partial \mathbf{h}^{(l)}_{u,c}}
\right)^{\!\top}
\boldsymbol{\delta}_{v,c}^{(l+1)}.
$$
The owner $c'$ of $u$ receives $\boldsymbol{\delta}_{u \leftarrow c}^{(l,\mathrm{part})}$ from every client $c\neq c'$ with $u\in V_c^{\mathrm{rem}}$ and accumulates its client-local adjoint
$$
\boldsymbol{\delta}_{u,c'}^{(l)}
=
\boldsymbol{\delta}_{u \leftarrow c'}^{(l,\mathrm{part})}
+
\sum_{c:\,u\in V_c^{\mathrm{rem}}}
\boldsymbol{\delta}_{u \leftarrow c}^{(l,\mathrm{part})}.
$$
No client backpropagates through layer $l-1$ until these partial adjoints have been returned and accumulated at their owners, so every owner starts the next step with the complete adjoint of its layer-$l$ representation.

During the backward pass through layer $l$, client $c$ accumulates the gradient of the layer-$l$ parameters generated by its owned outputs:
$$
g_c^{(l)}
=
\sum_{v\in V_c^{\mathrm{own}}}
\left(
\frac{\partial \mathbf{h}^{(l+1)}_{v,c}}
     {\partial \theta^{(l)}}
\right)^{\!\top}
\boldsymbol{\delta}_{v,c}^{(l+1)},
$$
and analogously for the input embedding,
$$
g_c^{\mathrm{emb}}
=
\sum_{v\in V_c^{\mathrm{own}}}
\left(
\frac{\partial \mathbf{h}^{(0)}_{v,c}}
     {\partial \theta^{\mathrm{emb}}}
\right)^{\!\top}
\boldsymbol{\delta}_{v,c}^{(0)}.
$$

\begin{theorem}[Gradient Equivalence]
\label{thm:gradient}
Under the shared-parameter assumption, for every client $c$, every owned node $v\in V_c^{\mathrm{own}}$, and every layer
$l\in\{0,\ldots,L\}$, the client-local adjoint accumulated at $c$ equals the centralized adjoint:
$
\boldsymbol{\delta}_{v,c}^{(l)}=\boldsymbol{\delta}_v^{(l,\mathrm{cent})}.
$
Then we have $\Gamma_{v,c}^{(l)}=0.$
Consequently, the parameter-gradient contributions accumulated
across clients equal the centralized parameter gradient:
$$
\sum_{c=1}^{C}g_c^{(l)}
=
\frac{\partial L^{\mathrm{cent}}}
     {\partial\theta^{(l)}},
\qquad
\sum_{c=1}^{C}g_c^{\mathrm{emb}}
=
\frac{\partial L^{\mathrm{cent}}}
     {\partial\theta^{\mathrm{emb}}}.
$$
\end{theorem}
\begin{proof}
Please see Appendix~\ref{app:gradient-equivalence-proof}.
\end{proof}

The two exchanges address different parts of FL: forward exchange recovers centralized inference, since inference depends only on the forward representations, while exact centralized training additionally requires backward exchange, since losses evaluated on one client may contribute to the gradient of representations owned by another. Under the shared-parameter assumption, the distributed forward/backward computation is therefore equivalent to centralized training. 

%% file: figures/exchange_figure.tex
\begin{figure}[t]
  \centering
  \resizebox{1.0\linewidth}{!}{%
    \begin{tikzpicture}[
        scale=0.9,
        >=latex,
        semithick,
        own/.style={
          circle,
          draw,
          minimum size=7.5mm,
          inner sep=0.5pt,
          font=\small
        },
        gho/.style={
          circle,
          draw,
          dashed,
          minimum size=7.5mm,
          inner sep=0.5pt,
          font=\small\itshape
        },
        client/.style={
          draw,
          rounded corners=4pt,
          inner sep=4.5mm
        },
        ie/.style={->},
        ce/.style={->, dashed, gray!65!black},
        ex/.style={->, thick, red!75!black},
        bex/.style={->, thick, blue!70!black}
    ]

    \begin{scope}
      \node[own, fill=blue!18]   (a) at (0,    0.8) {$a$};
      \node[own, fill=blue!18]   (b) at (1.5,  0.8) {$b$};
      \node[own, fill=orange!22] (c) at (1.5, -0.7) {$c$};
      \node[own, fill=orange!22] (d) at (0,   -0.7) {$d$};

      \draw[ie] (a) -- (b);
      \draw[ce] (b) -- (c);
      \draw[ie] (c) -- (d);
      \draw[ce] (d) -- (a);

      \node[font=\bfseries\small] at (0.75, -1.8)
        {(a) Global graph $G$};
    \end{scope}


    \begin{scope}[xshift=4.5cm]
      \node[own, fill=blue!18]  (a1) at (0,    0.8) {$a$};
      \node[own, fill=blue!18]  (b1) at (1.5,  0.8) {$b$};
      \node[gho, fill=orange!8] (c1) at (1.5, -0.7) {$c$};
      \node[gho, fill=orange!8] (d1) at (0,   -0.7) {$d$};

      \draw[ie] (a1) -- (b1);
      \draw[ce] (b1) -- (c1);
      \draw[ce] (d1) -- (a1);

      \node[
        client,
        fit=(a1)(b1)(c1)(d1),
        label={[font=\bfseries\small]above:Client 1}
      ] {};
    \end{scope}

    \begin{scope}[xshift=8.7cm]
      \node[gho, fill=blue!8]    (a2) at (0,    0.8) {$a$};
      \node[gho, fill=blue!8]    (b2) at (1.5,  0.8) {$b$};
      \node[own, fill=orange!22] (c2) at (1.5, -0.7) {$c$};
      \node[own, fill=orange!22] (d2) at (0,   -0.7) {$d$};

      \draw[ie] (c2) -- (d2);
      \draw[ce] (b2) -- (c2);
      \draw[ce] (d2) -- (a2);

      \node[
        client,
        fit=(a2)(b2)(c2)(d2),
        label={[font=\bfseries\small]above:Client 2}
      ] {};
    \end{scope}


    \draw[ex, bend left=22]
      (a1.north) to
      node[midway, above, font=\scriptsize]
      {$\mathbf{h}_{a,1}^{(l)}$}
      (a2.north);

    \draw[ex, bend left=22]
      (b1.north) to
      node[midway, above, font=\scriptsize]
      {$\mathbf{h}_{b,1}^{(l)}$}
      (b2.north);

    \draw[ex, bend left=22]
      (c2.south) to
      node[midway, below, font=\scriptsize]
      {$\mathbf{h}_{c,2}^{(l)}$}
      (c1.south);

    \draw[ex, bend left=22]
      (d2.south) to
      node[midway, below, font=\scriptsize]
      {$\mathbf{h}_{d,2}^{(l)}$}
      (d1.south);


    \draw[bex, bend left=22]
      (a2.north) to
      node[midway, below, font=\scriptsize]
      {$\boldsymbol{\delta}_{a \leftarrow 2}^{(l,\mathrm{part})}$}
      (a1.north);

    \draw[bex, bend left=22]
      (b2.north) to
      node[midway, below, font=\scriptsize]
      {$\boldsymbol{\delta}_{b \leftarrow 2}^{(l,\mathrm{part})}$}
      (b1.north);

    \draw[bex, bend left=22]
      (c1.south) to
      node[midway, above, font=\scriptsize]
      {$\boldsymbol{\delta}_{c \leftarrow 1}^{(l,\mathrm{part})}$}
      (c2.south);

    \draw[bex, bend left=22]
      (d1.south) to
      node[midway, above, font=\scriptsize]
      {$\boldsymbol{\delta}_{d \leftarrow 1}^{(l,\mathrm{part})}$}
      (d2.south);

    \node[font=\bfseries\small] at (7.35, -2.6)
      {(b) Distributed view with forward/backward exchange};

    \end{tikzpicture}%
  }

  \vspace{-0.05in}
  \caption{
  Layer-wise exchange on a 4-cycle
  ($a \to b \to c \to d \to a$) split across two clients.
  \textbf{(a)}~The global graph: node colors preview the
  partition (Client 1: blue, Client 2: orange); cross-client
  edges are dashed gray and intra-client edges are solid black.
  \textbf{(b)}~The distributed view: each client materializes
  solid copies of its \emph{owned} nodes and dashed copies of
  the \emph{remote} endpoints; the same node appears at the
  same coordinates across panels for visual continuity. Red and blue lines denote forward and backward exchange respectively.
  }
  \label{fig:exchange}

  \vspace{-0.15in}
\end{figure}

%% file: sections/5experiments.tex
\section{Experiments}
\label{sec:experiments}

\paragraph{Experimental settings.}
We use two different datasets in our experiments. First, we use the graph generator of \cite{egressy2023provably_powerful} to produce synthetic directed multigraphs with seven injected structural patterns as labeled subgraphs: directed cycles of length 2--6 (\textbf{C2}--\textbf{C6}), scatter-gather (\textbf{S-G}), and biclique (\textbf{B-C}), following \cite{Altman2023_RealisticAML}. 
For federated partitioning, we partition the global graph using two well-known community-aware strategies: \emph{Louvain} \cite{Blondel2008_Louvain} and \emph{METIS} \cite{Karypis1998_METIS}. They can induce qualitatively different cross-client connectivity profiles.

We also use the Small-HI configuration of the IBM Anti-Money Laundering (AML) dataset \cite{Altman2023_RealisticAML} to verify the generality of our framework: a large-scale, realistic financial transaction network with severe class imbalance (full statistics in Appendix~\ref{appendix:label-dist}).
To reflect realistic federated deployment, we partition clients via Louvain community detection on a weighted bank-interaction graph, while keeping banks indivisible.

We adopt the Multi-PNA model \cite{egressy2023provably_powerful} in our experiments because this is a state-of-the-art solution for both the multigraph pattern detection and the AML datasets used in our evaluation. Due to varying class imbalance, we report per-task PR-AUC for the minority class in each pattern as well as the macro-averaged minority-class PR-AUC across tasks.
In all our experiments, we use mini-batch training with stochastic neighborhood sampling \cite{hamilton2017_GraphSAGE}.
Further details are provided in Appendix~\ref{appendix:label-dist}.

\subsection{Pattern Detection Results}
\label{sec:main-results}

Table~\ref{tab:results_15_clients_macro} reports macro-averaged minority-class PR-AUC at 15 clients across training/exchange configurations built on Fully Local, FedAvg, and Sync-SGD (full per-pattern results in Appendix~\ref{appendix:full-results}).

\begin{table}[!htbp]
\centering
\caption{Macro-averaged minority-class PR-AUC (\%) at 15 clients under Louvain and METIS partitioning. LE: Layer-wise Exchange, Fwd: Forward, Bwd: Backward. $\Delta$ columns are absolute differences from the indicated reference: $\Delta$ Local = vs.\ fully local; $\Delta$ FedAvg = vs.\ FedAvg. The full per-pattern PR-AUC breakdown is given in Appendix~\ref{appendix:full-results} (Tables~\ref{tab:results_15_clients_louvain} and~\ref{tab:results_15_clients_metis}).}
\label{tab:results_15_clients_macro}
\resizebox{\columnwidth}{!}{%
\begin{tabular}{lcccccc}
\toprule
& \multicolumn{3}{c}{\textbf{Louvain}} & \multicolumn{3}{c}{\textbf{METIS}} \\
\cmidrule(lr){2-4} \cmidrule(lr){5-7}
\textbf{Method} &
\textbf{Macro} & \textbf{$\Delta$ Local} & \textbf{$\Delta$ FedAvg} &
\textbf{Macro} & \textbf{$\Delta$ Local} & \textbf{$\Delta$ FedAvg} \\
\midrule
Centralized learning
  & \grad{97.09} & &
  & \grad{97.09} & & \\
\midrule
Fully Local
  & \grad{70.61} & -- & --
  & \grad{69.42} & -- & -- \\
FedAvg
  & \grad{73.82} & $\bm{+3.21}$ & --
  & \grad{71.93} & $\bm{+2.51}$ & -- \\
Sync-SGD
  & \grad{92.24} & $\bm{+21.63}$ & $\bm{+18.42}$
  & \grad{88.22} & $\bm{+18.80}$ & $\bm{+16.29}$ \\
\midrule
OptimES
  & \grad{83.13} & $\bm{+12.52}$ & $\bm{+9.31}$
  & \grad{76.55} & $\bm{+7.13}$ & $\bm{+4.62}$ \\
\midrule
Fully Local + Fwd-LE
  & \grad{73.84} & $\bm{+3.23}$ & $\bm{+0.02}$
  & \grad{73.90} & $\bm{+4.48}$ & $\bm{+1.97}$ \\
FedAvg + Fwd-LE
  & \grad{84.22} & $\bm{+13.61}$ & $\bm{+10.40}$
  & \grad{81.96} & $\bm{+12.54}$ & $\bm{+10.03}$ \\
Sync-SGD + Fwd-LE
  & \grad{94.02} & $\bm{+23.41}$ & $\bm{+20.20}$
  & \grad{88.05} & $\bm{+18.63}$ & $\bm{+16.12}$ \\
\midrule
Fully Local + Fwd/Bwd-LE
  & \grad{70.98} & $\bm{+0.37}$ & $\bm{-2.84}$
  & \grad{69.39} & $\bm{-0.03}$ & $\bm{-2.54}$ \\
FedAvg + Fwd/Bwd-LE
  & \grad{84.75} & $\bm{+14.14}$ & $\bm{+10.93}$
  & \grad{83.18} & $\bm{+13.76}$ & $\bm{+11.25}$ \\
Sync-SGD + Fwd/Bwd-LE
  & \grad{94.25} & $\bm{+23.64}$ & $\bm{+20.43}$
  & \grad{93.15} & $\bm{+23.73}$ & $\bm{+21.22}$ \\
\bottomrule
\end{tabular}}
\end{table}
\paragraph{Partitioning induces a large gap; FedAvg barely closes it, and forward exchange alone does not.}
Partitioning drops macro PR-AUC from the centralized ceiling of $97.09\%$ to $70.61\%$ (Louvain) / $69.42\%$ (METIS) under fully local training, with every pattern affected and only marginal recovery under FedAvg ($73.82\%$/$71.93\%$). Adding forward exchange without parameter synchronization (Fully Local + Fwd-LE) matches plain FedAvg ($+0.02$/$+1.97$pp) but falls $8$--$10$pp short of FedAvg + Fwd-LE: without synchronized parameters, Theorem~\ref{thm:representation}'s equivalence does not hold, so exchanged embeddings no longer approximate the centralized ones.

\paragraph{Fresh, frequent synchronization closes most of the gap.}
FedAvg + Fwd-LE lifts macro PR-AUC to $84.22\%$ (Louvain) / $81.96\%$ (METIS), $+10.40$/$+10.03$pp over FedAvg, and outperforms the stale per-epoch OptimES by $1.09$/$5.41$pp --- a gap that widens under the denser METIS split, where stale embeddings are referenced more often per step. Increasing synchronization frequency closes most of what remains: Sync-SGD + Fwd-LE reaches $94.02\%$/$88.05\%$, a further $+9.80$/$+6.09$pp, since per-step parameter aggregation satisfies Theorem~\ref{thm:representation}'s condition not only at epoch boundaries, but continuously.

\paragraph{Backward exchange adds a further, structure-dependent gain, and mini-batch sampling caps the ceiling.}
Backward exchange improves every forward-only setting under parameter aggregation, most visibly under Sync-SGD on the denser METIS split ($88.05\%\to93.15\%$, $+5.10$pp, vs.\ only $+0.23$pp on Louvain), consistent with backward exchange mattering most when more gradient dependencies cross client boundaries; Sync-SGD + Fwd/Bwd-LE is the only configuration to beat plain Sync-SGD on both splits. The remaining gap to the centralized ceiling ($2.84$/$3.94$pp for Sync-SGD+LE) stems from mini-batch neighborhood sampling, since clients sample their $L$-hop neighborhoods independently and receptive fields are not perfectly aligned across clients as they would be centrally. These trends are robust as federation scale varies from $3$ to $15$ clients, with Sync-SGD+Fwd/Bwd-LE remaining the most robust configuration under dense (METIS) partitioning; see Appendix~\ref{appendix:client-sweep} for the full client-count sweep.

\begin{figure}[!tbp]
    \centering
    \includegraphics[width=0.9\linewidth]{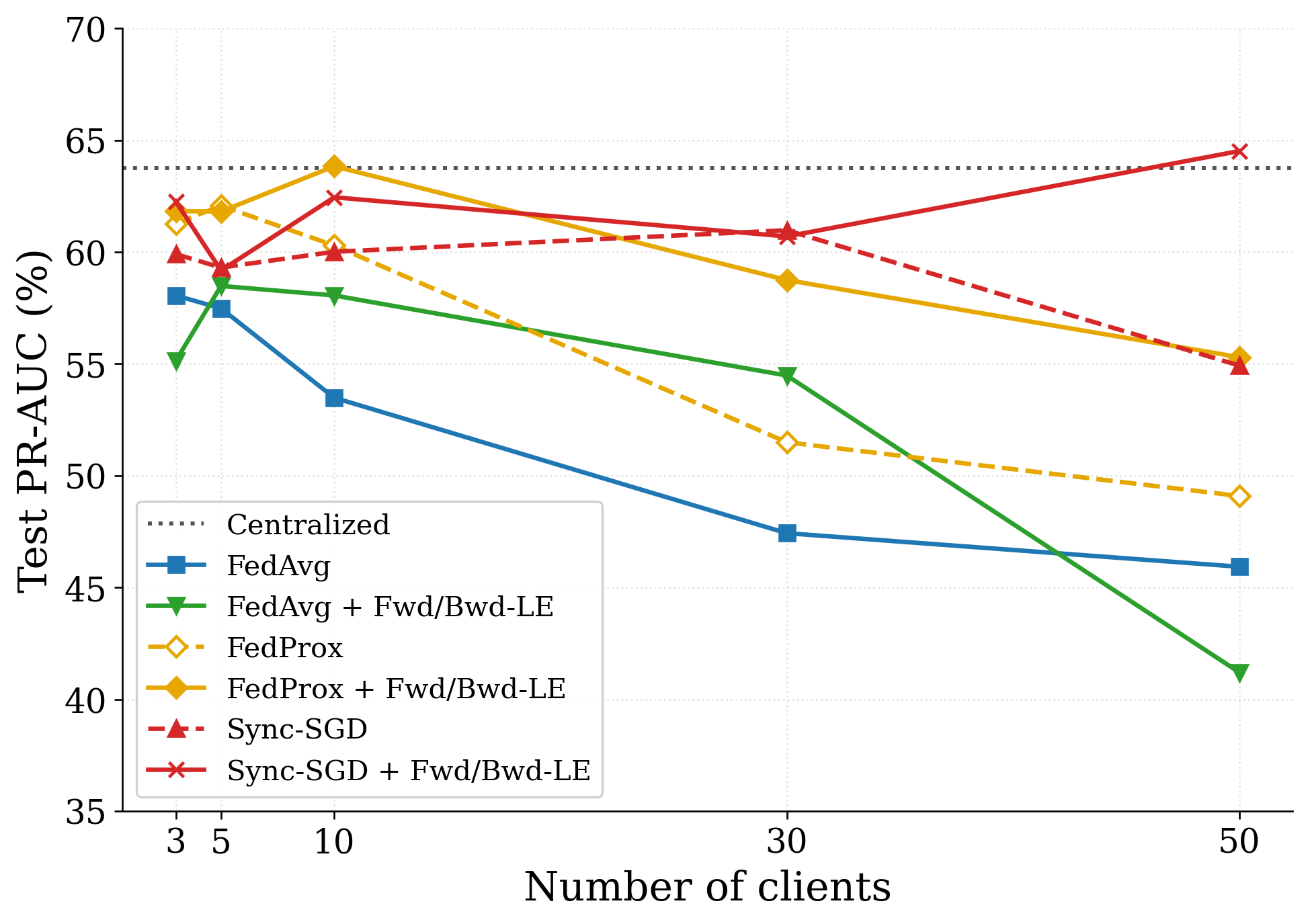}
    \vspace{-0.05in}
    \caption{Test PR-AUC on IBM's AML Small-HI dataset at the best validation round for different numbers of clients.}
    \label{fig:aml-client-scaling}
    \vspace{-0.1in}
\end{figure}
\subsection{AML Results}

Unlike synthetic pattern detection datasets, AML partitioning induces heterogeneous client topologies, transaction volumes, and label distributions. FedProx \cite{li2020federated} limits client drift through a proximal term, raising the question of whether this suffices for effective exchange without Sync-SGD’s per-step synchronization overhead. As Figure~\ref{fig:aml-client-scaling} shows, FedProx + Fwd/Bwd-LE consistently outperforms FedProx and closely tracks Sync-SGD + Fwd/Bwd-LE from 3 to 30 clients, falling clearly behind only at 50 clients. In contrast, FedAvg + Fwd/Bwd-LE is unstable and underperforms FedAvg at 3 and 50 clients, as looser parameter control allows exchanged representations and adjoints to drift further from validity. These results suggest that bounding client drift can preserve the benefits of layer-wise exchange without requiring per-step synchronization.

%% file: sections/6conclusion.tex
\section{Conclusion}
Our work represents a small but important step towards enabling cross-client pattern detection using distributed graph neural networks that learn what information to exchange across client boundaries. 
Notably, we prove that, under a shared-parameter assumption, the distributed model is equivalent to its centralized counterpart with layer-wise representation and gradient exchange. Our experiments show that parameter synchronization and layer-wise exchange are complementary and jointly essential for effective learning. Finally, on realistic money laundering detection datasets, we show that strict parameter synchronization is unnecessary and that constraining model divergence provides a more communication-efficient alternative. Future work will focus on strengthening privacy guarantees and further reducing communication requirements (discussed in Appendix~\ref{appendix:discussion}).

%% file: sections/appendix-full-problem-formulation.tex
\section{Message-Passing Updates}
\label{app:problem-formulation}

This appendix gives the explicit message-passing update used in Section~\ref{sec:problem-formulation} and in the proofs of Appendix~\ref{app:proof-of-theorems}.

\paragraph{Message-passing computation.}
Let $\Theta=\{\theta^{\mathrm{emb}}\}\cup\{\theta^{(l)}\}_{l=0}^{L-1}$ denote the model parameters. Let $\mathcal{N}(v)=\mathcal{N}^{\mathrm{in}}(v)\cup\mathcal{N}^{\mathrm{out}}(v)$ be the neighborhood used by the message-passing rule. The initial centralized representation is $\mathbf{h}_v^{(0,\mathrm{cent})}=\phi_{\theta^{\mathrm{emb}}}(\mathbf{x}_v)$. For $l\in\{0,\ldots,L-1\}$, the centralized layer update is
\begin{equation}
\mathbf{h}_v^{(l+1,\mathrm{cent})}
=
\phi_{\theta^{(l)}}\left(
\mathbf{h}_v^{(l,\mathrm{cent})},
\bigoplus_{u\in\mathcal{N}(v)}
\psi_{\theta^{(l)}}\left(
\mathbf{h}_v^{(l,\mathrm{cent})},
\mathbf{h}_u^{(l,\mathrm{cent})},
\mathbf{e}_{uv}
\right)
\right),
\label{eq:centralized-message-passing}
\end{equation}
where $\psi_{\theta^{(l)}}$ and $\phi_{\theta^{(l)}}$ are the message and node-update functions, respectively, and $\bigoplus$ is a permutation-invariant aggregation operator.

%% file: sections/appendix-proof-of-theorems.tex
\section{Proof of Theorems}
\label{app:proof-of-theorems}

\subsection{Proof for Representation Equivalence Theorem}
\label{app:representation-equivalence-proof}

\begin{proof}
By induction on the layer index $l$.

\textit{Base case ($l = 0$).}
For each client $c$ and each owned node $v \in V_c^{\mathrm{own}}$, the local initialization sets
\[
\mathbf{h}_{v,c}^{(0)} = \phi_{\theta^{\mathrm{emb}}}(\mathbf{x}_v).
\]
For each remote node $u \in V_c^{\mathrm{rem}}$, $u$ is owned by some client $c'$, which computes
$\mathbf{h}_{u,c'}^{(0)} = \phi_{\theta^{\mathrm{emb}}}(\mathbf{x}_u)$ locally; forward exchange then delivers this value to $c$, so
$\mathbf{h}_{u,c}^{(0)} = \phi_{\theta^{\mathrm{emb}}}(\mathbf{x}_u)$.
The centralized model applies the same input embedding under the shared-parameter assumption, hence
$\mathbf{h}_{u,c}^{(0)} = \mathbf{h}_u^{(0,\mathrm{cent})}$ for every $u \in V_c^{\mathrm{own}} \cup V_c^{\mathrm{rem}}$.

\textit{Inductive hypothesis.}
Assume that for some $l \geq 0$, for all clients $c$ and all nodes $u \in V_c^{\mathrm{own}} \cup V_c^{\mathrm{rem}}$,
\[
\mathbf{h}_{u,c}^{(l)} = \mathbf{h}_u^{(l,\mathrm{cent})}.
\]

\textit{Inductive step ($l \to l+1$).}
Let $v \in V_c^{\mathrm{own}}$ be any owned node. We show that client $c$ computes $\mathbf{h}_{v,c}^{(l+1)} = \mathbf{h}_v^{(l+1,\mathrm{cent})}$.

\textit{Step 1: Neighborhood availability.}
By the definition of $V_c^{\mathrm{rem}}$, every neighbor of $v$ in the global graph is present on client $c$ either as an owned or a remote node: $\mathcal{N}(v) \subseteq V_c^{\mathrm{own}} \cup V_c^{\mathrm{rem}}$, where $\mathcal{N}(v)=\mathcal{N}^{\mathrm{in}}(v)\cup\mathcal{N}^{\mathrm{out}}(v)$ is the (symmetric) neighborhood used by the message-passing rule (Appendix~\ref{app:problem-formulation}).

\textit{Step 2: Correctness of neighbor embeddings.}
For each neighbor $u \in \mathcal{N}(v)$:
\begin{itemize}
\item If $u \in V_c^{\mathrm{own}}$, then $\mathbf{h}_{u,c}^{(l)}$ is computed locally by client $c$ and is exact by the inductive hypothesis.
\item If $u \in V_c^{\mathrm{rem}}$, then $u$ is owned by some other client $c'$, which computes $\mathbf{h}_{u,c'}^{(l)} = \mathbf{h}_u^{(l,\mathrm{cent})}$ by the inductive hypothesis applied at client $c'$. By forward exchange, client $c$ receives this exact embedding, so $\mathbf{h}_{u,c}^{(l)} = \mathbf{h}_u^{(l,\mathrm{cent})}$.
\end{itemize}
In both cases, $\mathbf{h}_{u,c}^{(l)} = \mathbf{h}_u^{(l,\mathrm{cent})}$ for all $u \in \mathcal{N}(v)$.

\textit{Step 3: Layer update equivalence.}
By the shared-parameter assumption, client $c$ applies the same layer functions $\psi_{\theta^{(l)}}$ and $\phi_{\theta^{(l)}}$ as the centralized computation. Since the self-embedding $\mathbf{h}_{v,c}^{(l)}$, all neighbor embeddings $\mathbf{h}_{u,c}^{(l)}$, and the edge features $\mathbf{e}_{uv}$ are identical to the centralized setting, substituting into the layer-$l$ update \eqref{eq:centralized-message-passing} gives:
\[
\begin{aligned}
\mathbf{h}_{v,c}^{(l+1)}
&=
\phi_{\theta^{(l)}}\!\left(
\mathbf{h}_{v,c}^{(l)},\;
\bigoplus_{u \in \mathcal{N}(v)}
\psi_{\theta^{(l)}}\!\left(
\mathbf{h}_{v,c}^{(l)},\, \mathbf{h}_{u,c}^{(l)},\, \mathbf{e}_{uv}
\right)
\right) \\
&=
\phi_{\theta^{(l)}}\!\left(
\mathbf{h}_v^{(l,\mathrm{cent})},\;
\bigoplus_{u \in \mathcal{N}(v)}
\psi_{\theta^{(l)}}\!\left(
\mathbf{h}_v^{(l,\mathrm{cent})},\, \mathbf{h}_u^{(l,\mathrm{cent})},\, \mathbf{e}_{uv}
\right)
\right) \\
&= \mathbf{h}_v^{(l+1,\mathrm{cent})}.
\end{aligned}
\]

Since the partition is exhaustive ($\bigcup_c V_c^{\mathrm{own}} = V$), every node in $G$ is owned by exactly one client. By induction, for all $l \in \{0, \dots, L\}$ and all owned nodes $v \in V_c^{\mathrm{own}}$, $\mathbf{h}_{v,c}^{(l)} = \mathbf{h}_v^{(l,\mathrm{cent})}$.
\end{proof}

\subsection{Proof for Gradient Equivalence Theorem}
\label{app:gradient-equivalence-proof}

\begin{proof}
Backward exchange is the reverse-mode adjoint of a forward pass that already reproduces the centralized computation (Theorem~\ref{thm:representation}); we show by reverse induction on $l$ that the accumulated adjoint matches the centralized adjoint at every node and layer, from which the parameter-gradient identity follows directly.

\textit{Base case ($l = L$).} For each client $c$ and each owned node $v \in V_c^{\mathrm{own}}$, client $c$ computes its output adjoint from its own loss term, $\boldsymbol{\delta}^{(L)}_{v,c} = \partial \ell(\mathbf{h}^{(L)}_{v,c}, y_v) / \partial \mathbf{h}^{(L)}_{v,c}$. By the loss equivalence $\sum_c L_c = L^{\mathrm{cent}}$ established after Theorem~\ref{thm:representation}, node $v$ enters $L^{\mathrm{cent}}$ only through this same term $\ell(\cdot, y_v)$, so $\boldsymbol{\delta}^{(L),\mathrm{cent}}_v = \partial \ell(\mathbf{h}^{(L,\mathrm{cent})}_v, y_v)/ \partial \mathbf{h}^{(L,\mathrm{cent})}_v$. By Theorem~\ref{thm:representation}, $\mathbf{h}^{(L)}_{v,c} = \mathbf{h}^{(L,\mathrm{cent})}_v$, hence $\boldsymbol{\delta}^{(L)}_{v,c} = \boldsymbol{\delta}^{(L,\mathrm{cent})}_v$. The output adjoint is local to each owner, so no exchange is required at layer $L$.

\textit{Inductive hypothesis.}
Assume that for some $l+1 \le L$, every owned node $v \in V_c^{\mathrm{own}}$ has $\boldsymbol{\delta}^{(l+1)}_{v,c} = \boldsymbol{\delta}^{(l+1),\mathrm{cent}}_v$.

\textit{Inductive step ($l+1 \to l$).}
Each client $c$ backpropagates through layer $l$ using the complete adjoints of its owned outputs. For each $v \in V_c^{\mathrm{own}}$, the complete adjoint $\boldsymbol{\delta}^{(l+1)}_{v,c}$ equals $\boldsymbol{\delta}^{(l+1),\mathrm{cent}}_v$ by the inductive hypothesis. Applying reverse-mode through the layer-$l$ update yields, for each node $u \in V_c^{\mathrm{own}} \cup V_c^{\mathrm{rem}}$, the partial contribution
\[
\boldsymbol{\delta}_{u \leftarrow c}^{(l,\mathrm{part})}
=
\sum_{v \,\in\, V_c^{\mathrm{own}} \,\cap\, (\{u\}\cup\mathcal{N}(u))}
\left(
\frac{\partial \mathbf{h}^{(l+1)}_{v,c}}
{\partial \mathbf{h}^{(l)}_{u,c}}
\right)^{\!\!\top}
\boldsymbol{\delta}_{v,c}^{(l+1)}.
\]
By the inductive hypothesis and Theorem~\ref{thm:representation}, every Jacobian and adjoint above equals its centralized value, so $\boldsymbol{\delta}_{u \leftarrow c}^{(l,\mathrm{part})}$ is exactly the centralized recursion restricted to the terms whose dependent node $v$ is owned by $c$.

Summing these contributions over all clients via backward exchange, the owner $c'$ of $u$ obtains
\[
\boldsymbol{\delta}_{u,c'}^{(l)}
=
\sum_{\text{clients } c}\;
\sum_{\substack{v \in V_c^{\mathrm{own}} \\ v \,\in\, \{u\}\cup\mathcal{N}(u)}}
\left(\frac{\partial \mathbf{h}^{(l+1),\mathrm{cent}}_{v}}{\partial \mathbf{h}^{(l),\mathrm{cent}}_{u}}\right)^{\!\!\top}
\boldsymbol{\delta}^{(l+1),\mathrm{cent}}_v.
\]
Because the partition is exhaustive ($\bigcup_c V_c^{\mathrm{own}} = V$), the union over clients of $\{v \in V_c^{\mathrm{own}} : v \in \{u\}\cup\mathcal{N}(u)\}$ is simply $\{u\}\cup\mathcal{N}(u)$, which is exactly the index set of the centralized recursion for $u$ (Problem Formulation, Eq.~\eqref{eq:centralized-adjoint}, specialized to $u$). Hence
\[
\begin{aligned}
\boldsymbol{\delta}_{u,c'}^{(l)}
&=
\left(\frac{\partial \mathbf{h}^{(l+1),\mathrm{cent}}_u}{\partial \mathbf{h}^{(l),\mathrm{cent}}_u}\right)^{\!\!\top}
\boldsymbol{\delta}^{(l+1),\mathrm{cent}}_u + \sum_{w \in \mathcal{N}(u)}
\left(\frac{\partial \mathbf{h}^{(l+1),\mathrm{cent}}_w}{\partial \mathbf{h}^{(l),\mathrm{cent}}_u}\right)^{\!\!\top}
\boldsymbol{\delta}^{(l+1,\mathrm{cent})}_w \\
&= \boldsymbol{\delta}^{(l,\mathrm{cent})}_u.
\end{aligned}
\]
This completes the induction.

\textit{Parameter gradient.}
The parameter $\theta^{(l)}$ produces the layer-$(l+1)$ representations, so the centralized gradient is
\[
\frac{\partial L^{\mathrm{cent}}}{\partial \theta^{(l)}}
=
\sum_{v \in V}
\left(
\frac{\partial \mathbf{h}^{(l+1),\mathrm{cent}}_v}
{\partial \theta^{(l)}}
\right)^{\!\top}
\boldsymbol{\delta}^{(l+1,\mathrm{cent})}_v.
\]
Each node $v$ is produced by its owner $c$, which holds the complete adjoint $\boldsymbol{\delta}_{v,c}^{(l+1)} = \boldsymbol{\delta}^{(l+1,\mathrm{cent})}_v$ after backward exchange at layer $l+1$. Hence
\[
g_c^{(l)}
=
\sum_{v\in V_c^{\mathrm{own}}}
\left(
\frac{\partial \mathbf{h}^{(l+1)}_{v,c}}
{\partial\theta^{(l)}}
\right)^{\!\top}
\boldsymbol{\delta}_{v,c}^{(l+1)}
\]
equals the centralized terms restricted to $v \in V_c^{\mathrm{own}}$, and since the partition is exhaustive,
\[
\sum_c g_c^{(l)}
=
\frac{\partial L^{\mathrm{cent}}}{\partial \theta^{(l)}}.
\]
The same argument applied with the recovered layer-$0$ adjoint $\boldsymbol{\delta}_{v,c}^{(0)}$ gives $\sum_c g_c^{\mathrm{emb}} = \partial L^{\mathrm{cent}}/\partial \theta^{\mathrm{emb}}$, with $g_c^{\mathrm{emb}} = \sum_{v\in V_c^{\mathrm{own}}}(\partial \mathbf{h}^{(0)}_{v,c}/\partial\theta^{\mathrm{emb}})^{\!\top}\boldsymbol{\delta}_{v,c}^{(0)}$, so the recovered gradient covers all of $\Theta$.
\end{proof}

%% file: sections/appendix-experimental-setting.tex
\section{Extra Information on Experimental Settings}
\label{appendix:label-dist}

\subsection{Synthetic Multigraph Dataset}
\paragraph{Dataset.} The base topology is a random circulant-like multigraph with the same parameters as in \cite{egressy2023provably_powerful}. Separate graphs are generated for training, validation, and test splits, each with $8192$ nodes, average degree $6$, and no node or edge features.

Each node is labeled according to seven structural pattern detection tasks. Table~\ref{tab:pattern_distribution} summarizes the average positive-class prevalence of each pattern in the generated synthetic training dataset. The class prevalence remains consistent across the training, validation, and test splits. Cycle-4, Cycle-5, and Cycle-6 are the only tasks where the positive class is the majority (Table~\ref{tab:pattern_distribution}). This imbalance motivates our use of minority-class PR-AUC throughout the evaluation.

\begin{table}[h!]
\centering
\caption{Average positive-class prevalence of each structural pattern in the generated synthetic training dataset. The prevalence is consistent across the training, validation, and test splits.}
\vspace{0.5em}
\label{tab:pattern_distribution}

\resizebox{\columnwidth}{!}{%
\begin{tabular}{lccccccc}
\toprule
\textbf{Statistic} &
\textbf{C2} & \textbf{C3} & \textbf{C4} & \textbf{C5} & \textbf{C6} &
\textbf{S-G} & \textbf{B-C} \\
\midrule

Observed (\%)
  & 19.3 & 33.7 & 52.0 & 67.1 & 77.7 & 31.4 & 31.5 \\

\bottomrule
\end{tabular}
}
\end{table}

\paragraph{GNN model and training}
We adopt the Multi-PNA architecture of \cite{egressy2023provably_powerful}, which extends the Principal Neighborhood Aggregation (PNA) model \cite{corso2020pna} with reverse message passing, ego ID embeddings, and port numbering, a combination proven to be more expressive on directed multigraphs. All experiments use a 6-layer instance of this Multi-PNA model. Because the synthetic graphs carry no node or edge features, structural signal is supplied entirely by the ego and port embeddings, and we initialize node features as the constant vector $\mathbf{x}_v = \mathbf{1}$ for all $v$.

Each of the seven patterns corresponds to a binary node classification head, trained jointly in a multi-label setting.

Even though the representation equivalence (Theorem 1) assumes full-batch inference, this setting is not practical for large graphs. We therefore use mini-batch training with neighborhood sampling, a well-accepted approximation that is particularly effective in inductive settings \cite{hamilton2017_GraphSAGE}. We discuss the empirical impact of this approximation in Section~\ref{sec:main-results}.

\paragraph{Federated partitioning.}
\emph{Louvain} \cite{Blondel2008_Louvain} maximizes modularity and therefore, places partition boundaries along the graph's natural low-density seams, producing few cross-client edges per client. \emph{METIS} \cite{Karypis1998_METIS} minimizes edge-cut subject to a hard balance constraint on partition size; the balance requirement forces cuts through denser regions when natural communities are unevenly sized, yielding higher cross-client edge density. In both cases, we retain cross-client edges in the client subgraphs (see Section~\ref{sec:problem-formulation}). We evaluate both forward-only layer-wise exchange (Fwd-LE)
and joint forward/backward layer-wise exchange (Fwd/Bwd-LE)
on top of Fully Local, FedAvg, and Sync-SGD training.

\subsection{IBM Anti-Money Laundering (AML) Dataset}
\paragraph{Dataset and graph construction.}
To verify the generality of our method in AML applications, we conduct experiments on the Small-HI configuration of the IBM Anti-Money Laundering dataset \cite{Altman2023_RealisticAML}. This dataset is a realistic synthetic financial transaction network in which accounts are associated with banks and transactions are labeled as either legitimate or money laundering. We represent each bank account as a node and each transfer from a sending account to a receiving account as a directed edge. The directed multigraph contains 515,088 account nodes, 30,470 banks, and 5,078,345 transaction edges. Among these transactions, 5,177 are labeled as laundering, corresponding to a positive-class prevalence of only 0.102\%. 

During preprocessing, account identifiers, currencies, and payment formats are encoded to numerical identifiers. The edge features are the transaction timestamp, received amount, receiving currency, and payment format. Node features are initialized to a constant value, and structural information is additionally provided through port numbering, ego identifiers, and reverse message passing.

We use a chronological rather than a random transaction split for train, validation, and test sets. Transactions are grouped into complete 24-hour intervals, and two day boundaries are selected to approximate a 60/20/20 train/validation/test ratio. More precisely, we select the boundaries that minimize the maximum relative deviation from the target proportions while ensuring that all three subsets are non-empty. 
This derives 6 training days, 2 validation days, and 10 test days, corresponding to September 1--6, September 7--8, and September 9--18, 2022, respectively. Table \ref{tab:small-hi-statistics} shows specific statistics of the temporal split of the Small-HI dataset.

\begin{table}[t]
\centering
\small
\begin{tabular}{lrrr}
\toprule
Split & Transactions & Laundering & Positive rate (\%) \\
\midrule
Train      & 3,248,921 & 2,530 & 0.0779 \\
Validation &   965,524 & 1,036 & 0.107  \\
Test       &   863,900 & 1,611 & 0.186  \\
\midrule
Total      & 5,078,345 & 5,177 & 0.102  \\
\bottomrule
\end{tabular}
\caption{Statistics of the temporal split of the Small-HI dataset.}
\label{tab:small-hi-statistics}
\end{table}

The graph views are cumulative in time. The training graph contains only training-period transactions, the validation graph contains both training and validation transactions while loss and metrics are computed only on validation labels, and the test graph contains all transactions while evaluation is restricted to test labels. The same global temporal boundaries are used for both centralized and federated learning experiments.

\paragraph{Bank-level federated partitioning.}
Banks are treated as indivisible units, so all accounts belonging to the same bank are assigned to the same client. We first derive an undirected weighted bank-interaction graph from the complete transaction network, where each vertex represents a bank and the weight between two banks is the number of transactions exchanged between them. Louvain community detection is then applied with resolution 1.0. When the number of detected communities differs from the requested number of clients, communities are split or merged to obtain exactly $C$ non-empty client partitions.

We then ensure that every bank appearing as the source of a transaction is assigned to exactly one client. Finally, complete banks are greedily reassigned to balance the number of local training mini-batches. For a client $c$, this load is measured as
\[
    \left\lceil \frac{N^{\mathrm{train}}_c}{8192} \right\rceil,
\]
where $N^{\mathrm{train}}_c$ denotes the number of training transactions whose source bank is assigned to client $c$. 
The original Louvain helps in making a selection when the load is the same, and does not guarantee that the original partition will always remain unchanged. This procedure preserves bank ownership while reducing computational imbalance between clients.

Each client retains all transactions incident to one of its assigned banks, including transactions with accounts owned by other clients, but only transactions originating from an assigned bank are labeled and used in the client's objective --- incoming cross-bank transactions serve as graph context only. Consequently, every transaction label is assigned to exactly one client, while cross-client edges remain available for layer-wise representation exchange.

\paragraph{Experimental setup.}
We compare a centralized learning model with three distributed optimization strategies: FedAvg, FedProx \cite{li2020federated}, and Sync-SGD. The centralized learning model is trained on the complete transaction graph. For the federated experiments, we vary the number of clients over $k \in \{3,5,10,30,50\}$, allowing us to study how performance changes as the transaction graph becomes increasingly distributed.

FedAvg performs one local epoch per communication round followed by sample-weighted model aggregation. FedProx follows the same protocol but augments each local objective with a proximal regularizer using $\mu=0.01$. Sync-SGD instead processes one mini-batch per client at each optimization step and synchronizes the aggregated gradient after every step. All clients participate in every communication round. 

All setups use the same two-layer Multi-PNA edge-classification model, a mini-batch size of 8192, and two-hop neighborhood sampling with up to 100 neighbors per hop. We repeat each configuration using five random seeds. Owing to the severe class imbalance, model selection is based on validation PR-AUC, and the test PR-AUC from the corresponding best-validation round is reported. Final results are summarized using the mean and standard deviation across seeds.

%% file: sections/appendix-full-results.tex
\section{Extended Results}
\label{appendix:full-results}

\subsection{Full Per-Pattern Results}
For completeness, Tables~\ref{tab:results_15_clients_louvain} and~\ref{tab:results_15_clients_metis} report the full per-pattern PR-AUC breakdown underlying the macro-averaged results in Table~\ref{tab:results_15_clients_macro} (Section~\ref{sec:main-results}), covering all seven structural patterns at $15$ clients under Louvain and METIS partitioning, respectively.

\begin{table*}[!tb]
\centering
\caption{PR-AUC (\%) at 15 clients for different patterns under Louvain (left) and METIS (right) partitioning.}
\begin{subtable}[t]{0.48\textwidth}
\centering
\resizebox{\linewidth}{!}{%
\begin{tabular}{lccccccc}
\toprule
\textbf{Method} &
\textbf{C2} & \textbf{C3} & \textbf{C4} & \textbf{C5} & \textbf{C6} &
\textbf{S-G} & \textbf{B-C} \\
\midrule
Centralized learning
  & \grad{99.63} & \grad{99.60} & \grad{99.68} & \grad{93.51} & \grad{88.57} &
    \grad{99.34} & \grad{99.30} \\
\midrule
Fully Local
  & \grad{43.50}$\pm$0.17 & \grad{64.15}$\pm$0.25 & \grad{83.57}$\pm$0.15 & \grad{80.94}$\pm$0.18 & \grad{81.17}$\pm$0.16 &
    \grad{69.74}$\pm$0.39 & \grad{71.21}$\pm$0.20 \\
FedAvg
  & \grad{56.48}$\pm$4.29 & \grad{70.52}$\pm$1.86 & \grad{85.15}$\pm$0.45 & \grad{78.42}$\pm$0.58 & \grad{74.49}$\pm$1.81 &
    \grad{75.09}$\pm$0.75 & \grad{76.57}$\pm$0.16 \\
Sync-SGD
  & \grad{99.33}$\pm$0.00 & \grad{95.51}$\pm$0.07 & \grad{93.94}$\pm$1.00 & \grad{85.85}$\pm$0.00 & \grad{89.10}$\pm$0.03 &
    \grad{92.72}$\pm$2.13 & \grad{89.20}$\pm$3.84 \\
\midrule
OptimES 
  & \grad{95.42}$\pm$0.33 & \grad{84.06}$\pm$2.17 & \grad{86.83}$\pm$0.37 & \grad{84.59}$\pm$0.13 & \grad{87.29}$\pm$0.09 & \grad{72.28}$\pm$0.14 & \grad{71.46}$\pm$0.15 \\
\midrule
Fully Local + Fwd-LE
  & \grad{88.99}$\pm$0.36 & \grad{59.07}$\pm$0.03 & \grad{79.16}$\pm$0.07 & \grad{76.56}$\pm$0.26 & \grad{77.76}$\pm$0.42 &
    \grad{66.48}$\pm$0.61 & \grad{68.83}$\pm$0.13 \\
FedAvg + Fwd-LE
  & \grad{95.06}$\pm$0.80 & \grad{89.66}$\pm$0.95 & \grad{87.12}$\pm$1.00 & \grad{84.42}$\pm$0.14 & \grad{87.70}$\pm$0.21 &
    \grad{73.64}$\pm$0.39 & \grad{71.91}$\pm$0.42 \\
Sync-SGD + Fwd-LE
  & \grad{99.10}$\pm$0.16 & \grad{95.35}$\pm$0.69 & \grad{95.06}$\pm$0.21 & \grad{86.60}$\pm$0.24 & \grad{89.72}$\pm$0.01 &
    \grad{96.22}$\pm$0.35 & \grad{96.08}$\pm$0.19 \\
\midrule
Fully Local + Fwd/Bwd-LE
  & \grad{43.76}$\pm$0.27 & \grad{64.30}$\pm$0.20 & \grad{83.54}$\pm$0.03 &
    \grad{81.91}$\pm$0.07 & \grad{82.95}$\pm$0.34 &
    \grad{69.60}$\pm$0.06 & \grad{70.78}$\pm$0.34 \\
FedAvg + Fwd/Bwd-LE
  & \grad{97.02}$\pm$0.08 & \grad{90.77}$\pm$0.58 & \grad{86.58}$\pm$0.74 &
    \grad{84.53}$\pm$0.29 & \grad{87.72}$\pm$0.37 &
    \grad{74.57}$\pm$0.26 & \grad{72.08}$\pm$0.54 \\
Sync-SGD + Fwd/Bwd-LE
  & \grad{99.38}$\pm$0.10 & \grad{96.33}$\pm$0.19 & \grad{95.26}$\pm$0.03 &
    \grad{86.61}$\pm$0.43 & \grad{89.61}$\pm$0.27 &
    \grad{96.67}$\pm$0.28 & \grad{95.91}$\pm$0.22 \\
\bottomrule
\end{tabular}}
\vspace{2mm}
\caption{\textbf{Louvain} split.}
\label{tab:results_15_clients_louvain}
\end{subtable}\hfill
\begin{subtable}[t]{0.48\textwidth}
\centering
\resizebox{\linewidth}{!}{%
\begin{tabular}{lccccccc}
\toprule
\textbf{Method} &
\textbf{C2} & \textbf{C3} & \textbf{C4} & \textbf{C5} & \textbf{C6} &
\textbf{S-G} & \textbf{B-C} \\
\midrule
Centralized learning
  & \grad{99.63} & \grad{99.60} & \grad{99.68} & \grad{93.51} & \grad{88.57} &
    \grad{99.34} & \grad{99.30} \\
\midrule
Fully Local
  & \grad{48.96}$\pm$3.95 & \grad{63.76}$\pm$0.45 & \grad{80.00}$\pm$0.37 & \grad{78.06}$\pm$0.27 & \grad{76.73}$\pm$0.71 &
    \grad{68.68}$\pm$0.51 & \grad{69.72}$\pm$0.42 \\
FedAvg
  & \grad{61.05}$\pm$3.76 & \grad{68.51}$\pm$1.50 & \grad{80.48}$\pm$0.67 & \grad{76.71}$\pm$1.01 & \grad{71.32}$\pm$2.81 &
    \grad{71.41}$\pm$0.07 & \grad{74.06}$\pm$0.59 \\
Sync-SGD 
  & \grad{99.50}$\pm$0.06 & \grad{92.91}$\pm$0.88 & \grad{89.44}$\pm$0.51 & \grad{82.89}$\pm$0.10 & \grad{85.44}$\pm$0.19 &
    \grad{85.35}$\pm$1.21 & \grad{82.04}$\pm$0.32 \\
\midrule
OptimES 
  & \grad{76.69}$\pm$4.19 & \grad{69.44}$\pm$1.29 & \grad{81.91}$\pm$0.17 & \grad{80.55}$\pm$0.13 & \grad{83.62}$\pm$0.23 & \grad{71.44}$\pm$0.38 & \grad{72.23}$\pm$0.13 \\
\midrule
Fully Local + Fwd-LE
  & \grad{88.89}$\pm$0.18 & \grad{59.44}$\pm$0.07 & \grad{79.23}$\pm$0.25 & \grad{76.69}$\pm$0.12 & \grad{77.98}$\pm$0.15 &
    \grad{66.77}$\pm$0.42 & \grad{68.33}$\pm$0.04 \\
FedAvg + Fwd-LE
  & \grad{93.71}$\pm$0.08 & \grad{83.70}$\pm$3.62 & \grad{83.58}$\pm$0.02 & \grad{82.70}$\pm$0.14 & \grad{85.51}$\pm$0.06 &
    \grad{72.47}$\pm$0.31 & \grad{72.07}$\pm$0.12 \\
Sync-SGD + Fwd-LE
  & \grad{97.26}$\pm$0.54 & \grad{90.18}$\pm$0.94 & \grad{88.21}$\pm$0.34 & \grad{84.98}$\pm$0.13 & \grad{88.05}$\pm$0.02 &
    \grad{84.09}$\pm$1.44 & \grad{83.57}$\pm$0.37 \\
\midrule
Fully Local + Fwd/Bwd-LE
  & \grad{41.91}$\pm$0.17 & \grad{62.66}$\pm$0.06 & \grad{80.96}$\pm$0.06 &
    \grad{80.12}$\pm$0.49 & \grad{80.11}$\pm$0.66 &
    \grad{68.94}$\pm$0.09 & \grad{71.06}$\pm$0.03 \\
FedAvg + Fwd/Bwd-LE
  & \grad{94.84}$\pm$0.16 & \grad{87.79}$\pm$0.13 & \grad{84.55}$\pm$0.24 &
    \grad{83.05}$\pm$0.31 & \grad{86.13}$\pm$0.15 &
    \grad{73.87}$\pm$0.24 & \grad{72.01}$\pm$0.14 \\
Sync-SGD + Fwd/Bwd-LE
  & \grad{99.04}$\pm$0.18 & \grad{96.05}$\pm$0.57 & \grad{93.94}$\pm$0.56 &
    \grad{85.76}$\pm$0.09 & \grad{89.03}$\pm$0.24 &
    \grad{94.48}$\pm$1.03 & \grad{93.74}$\pm$0.15 \\
\bottomrule
\end{tabular}}
\vspace{2mm}
\caption{\textbf{METIS} split.}
\label{tab:results_15_clients_metis}
\end{subtable}
\end{table*}

\subsection{Robustness to Federation Scale}
\label{appendix:client-sweep}

\begin{figure*}[!tb]
  \centering
  \begin{subfigure}[t]{0.45\textwidth}
    \includegraphics[width=\linewidth]{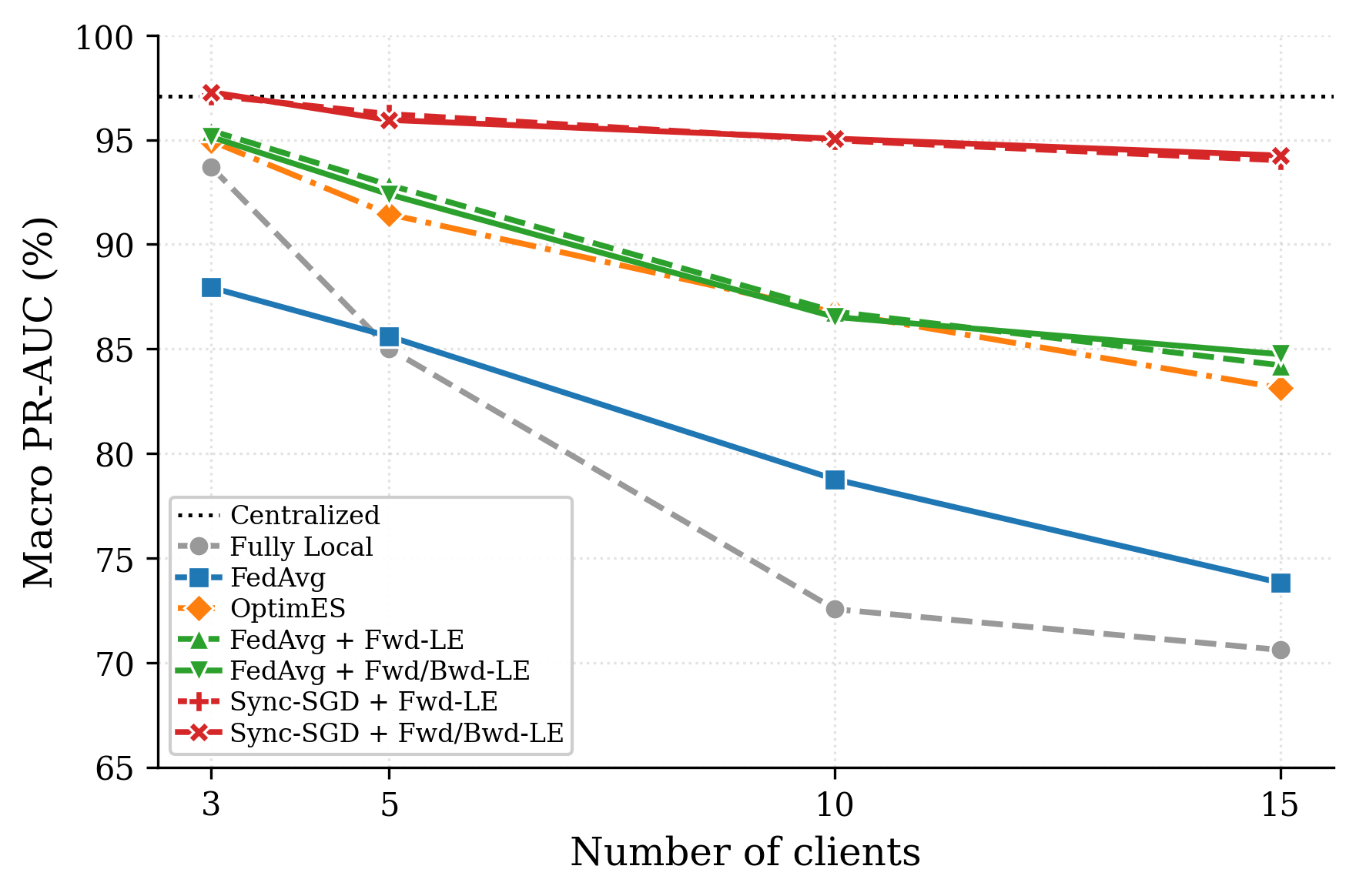}
    \caption{\textbf{Louvain} partitioning scheme.}
    \label{fig:macro_prauc_vs_clients_louvain}
  \end{subfigure}\hfill
  \begin{subfigure}[t]{0.45\textwidth}
    \includegraphics[width=\linewidth]{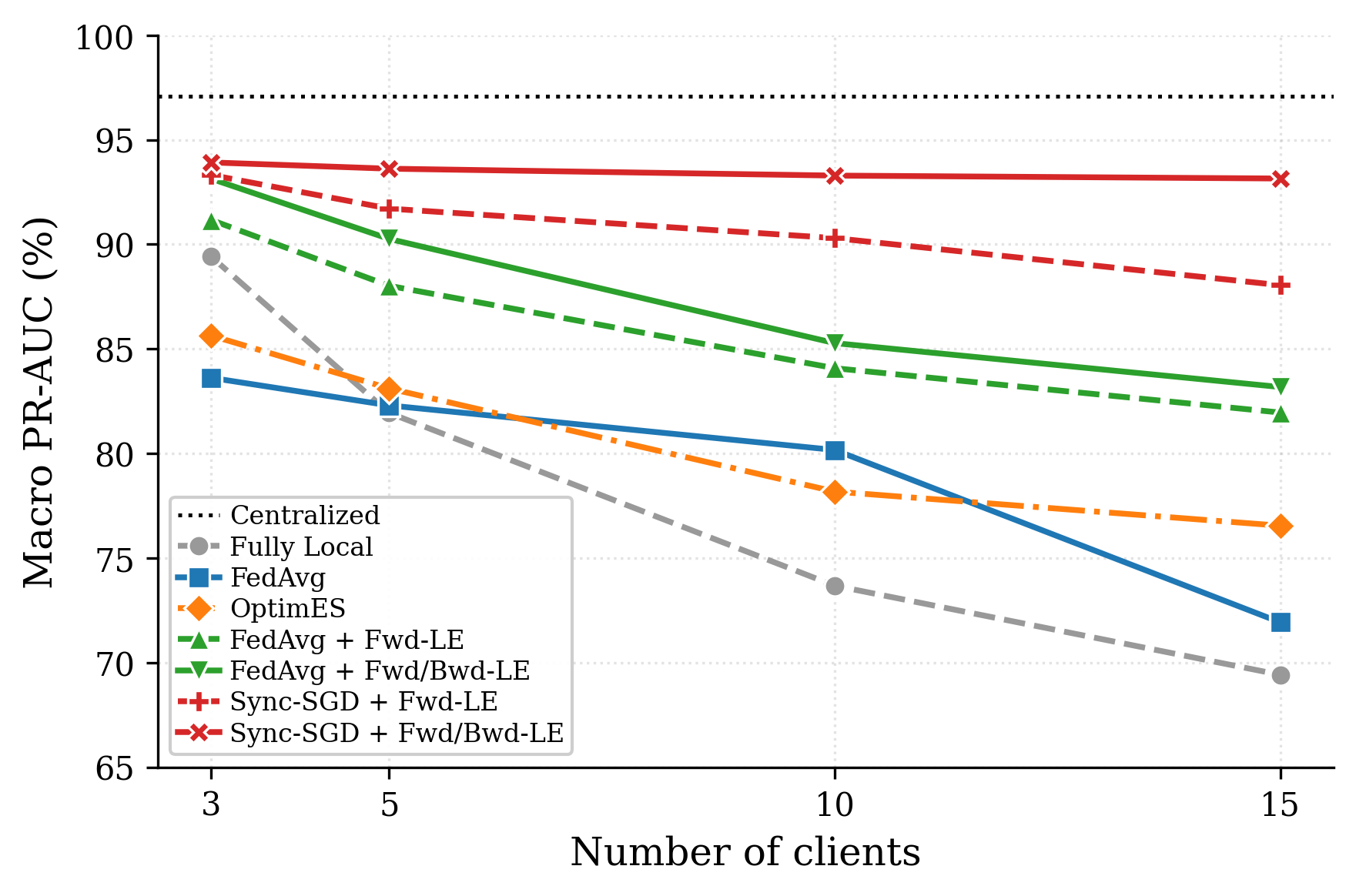}
    \caption{\textbf{METIS} partitioning scheme.}
    \label{fig:macro_prauc_vs_clients_metis}
  \end{subfigure}
  \caption{Macro PR-AUC on synthetic multigraph data as the federation scales from $3$ to $15$ clients under (a) Louvain and (b) METIS partitioning schemes. The dotted horizontal line indicates the centralized learning performance ($97.09\%$).}
  \label{fig:macro_prauc_vs_clients}
\end{figure*}

Figure~\ref{fig:macro_prauc_vs_clients} sweeps the number of clients, $k \in \{3,5,10,15\}$. Across both Louvain and METIS partitionings, all four exchange configurations built on FedAvg and Sync-SGD outperform Fully Local and FedAvg throughout the sweep. FedAvg+Fwd-LE consistently improves over FedAvg, while Sync-SGD+Fwd-LE further improves over FedAvg+Fwd-LE, with the performance gap widening as $k$ increases. Sync-SGD+Fwd-LE is also the least sensitive to federation scale, remaining nearly flat under both partitionings while Fully Local degrades sharply. Adding backward exchange preserves these trends. On Louvain, the Fwd/Bwd-LE variants closely track their forward-only counterparts across all client counts. On METIS, however, Fwd/Bwd-LE consistently outperforms Fwd-LE, with the gap increasing as $k$ grows. Consequently, Sync-SGD+Fwd/Bwd-LE is the most robust configuration under dense cross-client connectivity. The difference between fresh per-step exchange and stale per-epoch exchange follows the partition-density argument in Section~\ref{sec:main-results}. Per-layer freshness therefore becomes increasingly important as cross-client connectivity densifies, irrespective of federation scale.

%% file: sections/appendix-discussion.tex
\section{Discussion}
\label{appendix:discussion}

\paragraph{Communication cost analysis.}
Table \ref{tab:comm-cost-per-method} summarizes the per-epoch communication volume of each evaluated method for a single client (client $c$). FedAvg transmits the parameter vector once at each epoch boundary and therefore incurs only the parameter communication cost $P$. OptimES performs a single embedding exchange per epoch across all $L$ layers, activating each cross-client edge once and contributing an embedding communication term of $L\,d\,|E_c^{\mathrm{rem}}|$. FedAvg + Fwd-LE replaces this one-time exchange with a per-step exchange over the cross-client edges activated during mini-batch training. Consequently, the embedding term changes from $L\,d\,|E_c^{\mathrm{rem}}|$ to $S\,L\,d\,|R_c|$.

\begin{table}[!ht]
\centering
\caption{Per-epoch communication volume of each method.}
\label{tab:comm-cost-per-method}
\begin{tabular}{lc}
\toprule
\textbf{Method} & \textbf{Per-epoch volume (client $c$)} \\
\midrule
FedAvg                 & $P$ \\
OptimES                & $P + L\,d\,|E_c^{\mathrm{rem}}|$ \\
FedAvg + Fwd-LE        & $P + S\,L\,d\,|R_c|$ \\
FedAvg + Fwd/Bwd-LE    & $P + 2\,S\,L\,d\,|R_c|$ \\
Sync-SGD               & $S\,P$ \\
Sync-SGD + Fwd-LE      & $S\,P + S\,L\,d\,|R_c|$ \\
Sync-SGD + Fwd/Bwd-LE  & $S\,P + 2\,S\,L\,d\,|R_c|$ \\
\bottomrule
\end{tabular}
\end{table}

In contrast, Sync-SGD synchronizes the full parameter vector at every optimization step, multiplying the parameter communication term by $S$ and yielding a cost of $S\,P$. The corresponding forward/backward exchange variants additionally transmit adjoints, adding a backward exchange of equal size and thereby doubling the exchange volume to $2\,S\,L\,d\,|R_c|$ in the layer-wise exchange variants.

Fully Local, Fully Local + Fwd-LE, and Fully Local + Fwd/Bwd-LE are omitted from the analysis because they aggregate no parameters across clients. Their communication volumes are $0$, $S\,L\,d\,|R_c|$, and $2\,S\,L\,d\,|R_c|$, respectively.

\paragraph{Privacy risk and mitigation.}
Although our method does not directly transmit raw features, labels, or local topology, the exchanged information may still encode local features, neighborhoods, and model states. A semi-honest client may therefore attempt to infer private attributes, membership information, or structural properties from the received information. Backward exchange further expands this attack surface because the transmitted adjoints are associated with gradient signals and may be vulnerable to gradient inversion \cite{huang2021evaluating} and attribute-inference attacks \cite{jayaraman2022attribute}. Several mitigation strategies could reduce these risks: Homomorphic Encryption (HE) could protect both the representations and adjoints by keeping them encrypted during exchange; under threshold HE, clients jointly hold the decryption key so no single party can decrypt exchanged values alone \cite{mouchet2021multiparty}.
Existing HE-based GNN systems have demonstrated that graph inference can be evaluated over encrypted features by expressing message passing as HE-compatible arithmetic operations and replacing nonlinear activations with polynomial approximations \cite{ran2022cryptogcn,peng2023lingcn,ran2023penguin}.
These techniques provide the main building blocks for protecting our layer-wise exchange. Boundary-node representations could remain encrypted during the forward pass, while encrypted adjoints could be accumulated during backward exchange.
Secure aggregation~\cite{bonawitz2017practical} can further restrict access to authorized participants.
Local or global differential privacy noise could be applied to both the exchanged representations and adjoints~\cite{abadi2016deep,dwork2025differential}, yet at the cost of reduced utility.

\paragraph{Further reducing communication costs and privacy leakage.}
One possible way to reduce both communication cost and privacy leakage is to exchange representations and adjoints selectively --- restricting exchange to boundary regions likely to participate in cross-client patterns, rather than exchanging universally at every layer.
Node selection could be guided by structural heuristics, anomaly scores, learned attention weights, or uncertainty estimates.
Another possibility is clustered or hierarchical exchange, where clients are grouped by transaction volume, geographic region, business relationship, or cross-client connectivity, with frequent exchange restricted within each cluster.

\paragraph{Implications for real-world cross-client pattern detection.}
In realistic applications such as anti-money laundering (AML), transaction graphs are distributed across financial organizations, institutions, or jurisdictions, and our framework restores the missing message-passing signals at partition boundaries. However, deploying it in real-world AML systems would require additional infrastructure beyond the algorithmic mechanism itself.
Importantly, real financial graphs are dynamic: transactions arrive continuously and suspicious patterns evolve over time.
Moreover, real-world data heterogeneity is likely to be substantially stronger than in our synthetic benchmark, since institutions can differ in customer populations, transaction volumes, regulatory environments, and local risk standards.
We therefore view extending our framework to dynamic, heterogeneous, and operationally constrained AML environments as an important direction for future work.

\paragraph{Societal impact.}
Money laundering is often connected to broader criminal activity, including fraud, corruption, human trafficking, drug trafficking, weapons trafficking, and terrorist financing~\cite{safdari2015social}.
More effective detection tools can help institutions and regulators identify suspicious activity that would remain invisible from any single local view.
At the same time, such systems must be deployed carefully, since financial data is highly sensitive and false accusations can harm individuals and organizations.
The goal is therefore not only improving detection accuracy, but also designing systems that balance effectiveness, privacy, fairness, auditability, and regulatory compliance.
Our work contributes to this direction by identifying structural observability as a core challenge and by showing how cross-client representation exchange can help uncover financial crime patterns that span multiple organizations, institutions, or jurisdictions.